\documentclass[letterpaper,11pt]{article}

\usepackage{amsthm}
\usepackage{amsmath}
\usepackage{graphicx}
\usepackage{algorithm}
\usepackage{algorithmic}
\usepackage{fullpage}
\usepackage{setspace}
\usepackage{listings}

\usepackage{natbib}

\newtheorem{thm}{Theorem}[section]

\newtheorem{lem}[thm]{Lemma}

\theoremstyle{definition}

\theoremstyle{remark}

\numberwithin{equation}{section}

\newcommand{\eps}{\varepsilon}

\newcommand{\ignore}[1]{}

\def\one{{\bf 1}}

\def\OPT{\operatorname{OPT}}

\def\argmin{\operatorname{argmin}}

\def\R{{\bf R}}

\def\C{{\mathcal C}}
\def\cost{\operatorname{cost}}
\def\XX{{\cal X}}
\def\dd{{\bf d}}

\begin{document}
\title{Active Learning of Custering with Side Information Using $\eps$-Smooth Relative Regret Approximations}
\author{Nir Ailon and Ron Begleiter}

\setcounter{page}{0}
\maketitle
\thispagestyle{empty}
\maketitle
\begin{abstract}

Clustering is considered a non-supervised learning setting, in which the goal is to partition  a collection of data points
into disjoint clusters.  Often a bound $k$ on the number of clusters is given or assumed by the practitioner. 
Many versions of this problem have been defined, most notably $k$-means and $k$-median.

An underlying problem with the unsupervised nature of clustering it that of determining a similarity
function.
One approach for alleviating this difficulty is known as clustering with side information, alternatively, semi-supervised
clustering.  Here, the practitioner  incorporates  side information
in the form of ``must be clustered'' or ``must be separated'' labels for  data point pairs.  Each such piece of information
comes at a ``query cost'' (often involving human response solicitation).  The collection of labels
is then incorporated in the usual clustering algorithm as either strict or  as soft constraints, possibly adding 
a pairwise constraint penalty function to the chosen clustering objective.

Our work is mostly related to clustering with side information.  
We ask how to choose  the pairs of data points.  Our analysis gives rise
to a method provably better than simply choosing them uniformly at random.
Roughly speaking, we show that the distribution must be biased so as more weight is placed on pairs   incident to elements in smaller clusters in some optimal solution.
Of course we do not know the optimal solution, hence we don't know the bias.  Using the recently introduced
method of $\eps$-smooth relative regret approximations of Ailon, Begleiter and Ezra, 
 we can show an iterative process that improves both the
clustering and the  bias in tandem.  The process provably converges to the optimal 
solution faster (in terms of query cost) than an algorithm selecting pairs uniformly.



\end{abstract}
\section{Introduction}

Clustering of data is probably the most important problem in the theory of unsupervised learning.
In the most standard setting, the goal is to paritition  a collection of data points
into  related groups.
Virtually any large scale application using machine learning either uses clustering as a data preprocessing step
 or as an ends within itself.  

In the most tranditional sense, clustering is an unsupervised learning problem because the solution is computed
from the data itself, with no human labeling involved.  There are many versions, most notably  $k$-means and $k$-median.
The number  $k$ typically serves as an assumed upper bound on the number of output clusters.

An underlying difficulty with the unsupervised nature of clustering is the fact that a similarity (or distance) function
between data points must be chosen by the practitioner as a preliminary step.  This may often not be an easy task. 
Indeed, even if our dataset is readily embedded in some natural  vector (feature) space, we still have the burden
of the freedom of choosing a normed metric, and of applying some transformation (linear or otherwise) on the data for good measure.
Many approaches  have been proposed to tacle this.  In one approach, a  metric learning algorithm is executed as a preprocessing step in order 
to choose a suitable metric (from some family).  This approach is supervised, and  uses distances between pairs
of elements as  (possibly noisy) labels.  
The second approach is known as clustering with side information, alternatively, semi-supervised
clustering.  This approach should be thought of as adding crutches to a lame distance function 
 the practitioner is too lazy to replace.
Instead, she incorporates so-called side information
in the form of ``must be clustered'' or ``must be separated'' labels for  data point pairs.   Each such label
comes at a ``query cost'' (often involving human response solicitation).  The collection of labels
is then incorporated in the chosen clustering algorithm as either strict constraints or  as soft ones, possibly adding 
a pairwise constraint penalty function.

\subsection{Previous Related Work}
Clustering with side information is a
fairly new variant of clustering first described, independently, by
\citet{Demiriz99semi-supervisedclustering}, and \citet{Ben-DorSY99}.
In the machine learning community it is also widely known as \emph{semi-supervised clustering}.
There are a few alternatives for the form of feedback
providing the side-information. The most natural ones are the
single item labels
\citep[e.g.,][]{Demiriz99semi-supervisedclustering}, and the pairwise
constraints \citep[e.g.,][]{Ben-DorSY99}. 

In our study, the side information is pairwise,  comes at a cost and is treated frugaly.
In a related yet different setting,  similarity information for all (quadratically many) pairs
is available but is noisy.  The combinatorial optimization theoretical problem of cleaning
the noise is known as \emph{correlation clustering}  \citep{Bansal02correlationclustering} or 
\emph{cluster editing} \citep{Shamir:2004:CGM}.  Constant factor approximations are known
for various versions of this problems \citep{CharikarW04, ACN08}.  A PTAS is known
for a minimization version in which the number of clusters is fixed \citep{GiotisGuruswami06}.

%


Roughly speaking, there are two main approches for utilizing pairwise side information.
In the first approach, this information is used to fine tune or learn 
a \emph{distance} function, which is then passed on to any standard clustering algorithm.
Examples include \citet{Cohn03semi-supervisedclustering},
\citet{Klein02frominstance-level}, and \citet{Xing02distancemetric}.
The second approach, which is the starting point to our work, modifies  the clustering algorithms's objective
so as to incorporate the pairwise constraints.
\citet{basu05} in his thesis, which also serves as a comprehensive
survey, has championed this approach in conjunction with $k$-means,
and hidden Markov random field clustering algorithms.


\subsection{Our Contribution}
Our main motivation is reducing the number of pairwise similarity labels (query cost) required for
$k$-clustering data using an \emph{active learning} approach.
More precisely, we ask how to choose  which pairs of data points to query.  Our analysis gives rise
to a method provably better than simply choosing them uniformly at random.
More precisely, we show that the distribution from which we should draw pairs from must be biased so as more weight is placed on pairs   incident to elements in smaller clusters in some optimal solution.
Of course we do not know the optimal solution, let alone the  bias.  Using the recently introduced
method of $\eps$-smooth relative regret approximations ($\eps$-SRRA) of \citet{AilonBE11}
 we can show an iterative process that improves both the
clustering and the  bias in tandem.  The process provably converges to the optimal 
solution faster (in terms of query cost) than an algorithm  uniformly selecting pairs.
Optimality here is with respect to the (complete) pairwise constraint penalty function.

In Section~\ref{sec:notation} we define our problem mathematically.  
We then present  the $\eps$-SRRA method of \citet{AilonBE11} for the purpose of self containment in Section~\ref{sec:srra}.
Finally, we present our main result in Section~\ref{sec:srra_for_cc}.

\section{Notation and Definitions}\label{sec:notation}
Let $V$ be a set of points of size $n$.  Our goal is to partition $V$ into $k$ sets (clusters).
There are two sources of information guiding us in the process.  One is unsupervised, possibly emerging from
features attached to each element $v\in V$ together with a chosen distance function.  This information is captured in a utility function such as $k$-means
or $k$-medians.  The other type is supervised, and is encoded as an undirected graph $G=(V,E)$.
An edge $(u,v)\in E$
corresponds to the constraint \emph{$u$,$v$ should be clustered together} and a nonedge $(u,v)\not \in E$ corresponds to the converse.
Each edge or nonedge comes at a \emph{query cost}.  This means that $G$  exists only implicitly.  We uncover the
truth value of the predicate ``$(u,v)\in E$'' for any chosen pairs $u,v$ for a price.
We also assume that $G$ is riddled with human errors, hence it does not necessarily encode a perfect $k$ clustering of the data.
In what follows, we assume $G$ fixed.

A $k$-clustering $\C = \{C_1,\dots, C_k\}$ is a collection of  $k$ disjoint (possibly empty) sets satisfying $\bigcup C_i = V$.
We use the notation $u \equiv_\C v$ if $u$ and $v$ are in the same cluster, and $u \not \equiv_\C v$ otherwise.  

The cost of $\C$ with respect to $G$  is defined as
$$ \cost(\C) = \sum_{ (u,v)\in E} \one_{u \not \equiv_\C v} + \sum_{(u,v)\not \in E} \one_{u \equiv_\C v}\ .$$

Minimizing $\cost(\C)$ over clusterings when $G$ is known as correlation clustering (in complete graphs).
This problem was defined by \citet{BBC04} and has received much attention since (e.g. \cite{ACN08,CharikarGW05,5272169}). \footnote{The original problem definition did not limit the number of output clusters.}
\citet{5272169} achieved a PTAS for this problem, namely, a polynomial time algorithm returning a $k$-clustering
with cost at most $(1+\eps)$ that of the optimal.\footnote{The polynomial degree depends on $\eps$.} 
 Their PTAS is not query efficient:  It requires knowledge of $G$ in its entirety.
In this work we study the query complexity required for achieving a $(1+\eps)$ approximation for $\cost$.  From 
a learning theoretical perspective, we want to find the best $k$-clustering explaining $G$ using as few queries
as possibly into $G$.

\section{The $\eps$-Smooth Relative Regret Approximation  ($\eps$-SRRA) Method}\label{sec:srra}
Our search problem can be cast as a special case of the following more general learning problem. 
Given some possibly noisy structure (e.g. a graph in our case) $h$, the goal is to find the best explanation  using a limited space $\XX$ of hypothesis (in our case $k$-clusterings).  The goal is to minimize a notion of a nonnegative \emph{cost} which is defined as the distance $\dd(f, h)$ between
$f\in \XX$ and $h$.  Assume also that the distance function $\dd$ between $\XX$ and $h$ is an extension of a metric on $\XX$.
 \cite{AilonBE11} have recently shown the following general scheme for finding the best $f\in \XX$.  To explain this scheme, we need to define a notion
of $\eps$-smooth relative regret approximation.

Given a solution $f\in \XX$ (call it the \emph{pivotal} solution) and another solution $g\in \XX$, we define $\Delta_f(g)$ to be $\dd(g, h) - \dd(f, h)$, namely, the difference between the cost of the solution $g$ and the cost of the solution $f$. 
We call this the \emph{relative regret} function with respect to $f$.
 Assume we have oracle access  to a function $\hat \Delta_f : \XX \rightarrow \R$ such that for all $g\in \XX$,
$$ |\hat \Delta_f(g) - \Delta_f(g)| \leq \eps \dd(f, g)\ .$$
If such an estimator function $\hat \Delta_f$ exists, we say that it is an $\eps$-smooth regret approximation ($\eps$-SRRA) for with respect to $f$.  \citet{AilonBE11} show
that if we have an $\eps$-smooth regret approximation function, then it is possible to obtain a $(1+\eps)$-approximation to the optimal solution by repeating the iterative process presented in Figure~\ref{cd:alg}.
\begin{figure}
\fbox{
\parbox[l]{\linewidth}{
\begin{itemize}
\item Start with any solution $f_0\in \XX$
\item Set $t\leftarrow 0$
\item Repeat until some stopping condition:
\begin{itemize}
\item Set $f_{t} \leftarrow \argmin_{g\in \XX} \hat \Delta_{f_{t-1}}(g)$, where $\hat \Delta_{f_{t-1}}$ is an $\eps$-SRRA for $f_{t-1}$.
\item Set $t\leftarrow t+1$
\end{itemize}
\end{itemize}
}
}
\label{cd:alg}\caption{Iterative algorithm using $\eps$-SRRA}
\end{figure}
It is shown that this search algorithm converges \emph{exponentially fast} to an $(1+\eps)$-approximately optimal one.
More precisely, the following is shown:
\begin{thm}\citep{AilonBE11}\label{cd:thm:1}
	Assume input parameters $\eps \in (0,1/5)$  and initializer $\hat{f}_0\in \XX$ of Algorithm~\ref{cd:alg}. Denote 
	$\OPT := \min_{f\in\XX}\dd(f,h)$.  
	$\hat{f}_0\in \XX$ be an arbitrary function.
	Then the following holds for $\hat{f}_t$ obtained in Algorithm~\ref{cd:alg} for all $t \geq 1$:
	\begin{align}
		\dd( \hat{f}_t, h ) &\leq \left( 1 + 8\eps\right) \left(1 + (5\eps)^{t} \right) \OPT + 
							(5 \eps)^t \dd(\hat{f}_0, h). \label{cd:thm:1:eq}
	\end{align}
\end{thm} \noindent
  There are two questions now: (1)
How  can we build $\hat \Delta_f$  efficiently?  
(2) How do we find $ \argmin_{g\in \XX} \hat \Delta_{f}(g)$?

In the case of $k$-clusterings, the target structure $h$ is the graph $G$ and $\XX$ is the space of $k$-clusterings over $V$.
The metric $\dd$ over $\XX$ is taken to be
$$ d(\C,\C') = \frac 1 2 \sum_{u,v} d_{u,v}(\C, \C')\ $$
where $d_{u,v}(\C, \C') = \one_{u\equiv_{\C'} v}\one_{u\not \equiv_{\C} v} + \one_{u\equiv_{\C} v}\one_{u\not \equiv_{\C'} v}$.
By defining $\dd(\C, G) := \cost(\C)$ we clearly extend $\dd$  to a metric over $\XX\cup\{G\}$.

\section{$\eps$-Smooth Regret Approximation for $k$-Correlation Clustering}\label{sec:srra_for_cc}
Denote $ \cost_{u,v}(\C) = \one_{(u,v)\in E}\one_{u \not \equiv_\C v} + \one_{(u,v)\not \in E}\one_{u \equiv_\C v}\ ,$
so that $\cost(\C) =\frac 1 2 \sum_{u,v} \cost_{u,v}(\C)$.
Now consider another clustering $\C'$.  We are interested in the change in cost incurred by replacing $\C$ by $\C'$, in other words in the function $f$ defined as
$$ f(\C') = \cost(\C') - \cost(C)\ .$$

We would like to be able to compute an approximation $\hat f$ of $f$ by viewing only a sample of edges in $G$.  That is, we imagine that each edge query from $G$ costs us one unit, and we would like to reduce that cost while sacrificing our accuracy as little as possible.  
We will refer to the cost incurred by queries as the \emph{query complexity}.
Consider the following metric on the space of clusterings:
$$ d(\C,\C') = \frac 1 2 \sum_{u,v} d_{u,v}(\C, \C')\ $$
where $d_{u,v}(\C, \C') = \one_{u\equiv_{\C'} v}\one_{u\not \equiv_{\C} v} + \one_{u\equiv_{\C} v}\one_{u\not \equiv_{\C'} v}$.
(The distance function  simply counts the number of unordered pairs on which $\C$ and $\C'$ disagree on.)
Before we define our sampling scheme, we slightly reorganize the function $f$. 
Assume that $|C_1| \geq |C_2| \geq \cdots \geq |C_k|$.  Denote $|C_i|$ by $n_i$.
 The function $f$ will now be written as:
\begin{equation}\label{aaa}
f(\C') = \sum_{i=1}^k \sum_{u \in C_i} \left ( \frac 1 2 \sum_{v\in C_i} f_{u,v}(\C')  + \sum_{j=i+1}^k \sum_{v\in C_j} f_{u,v}(\C')   \right )
\end{equation}
where
$$ f_{u,v}(\C')= \cost_{u,v}(\C') - \cost_{u,v}(C)\ .$$ 

Note that $f_{u,v}(\C')\equiv 0$ whenever $\C$ and $\C'$ agree on the pair $u,v$.  For each $i\in [k]$, let $f_i(\C')$ denote the sum running over $u\in \C_i$ in (\ref{aaa}), so that $f(\C') = \sum f_i(\C')$.
Similarly, we now rewrite $d(\C, \C')$ as follows:
\begin{equation}
d(\C, \C') = \sum_{i=1}^k \sum_{u \in C_i} \left ( \sum_{j=i+1}^k \sum_{v\in C_j} \one_{u \equiv_{\C'} v} + \frac 1 2 \sum_{v\in C_i} \one_{u \not\equiv_{\C'} v} \right )
\end{equation}
and denote by $d_i(\C')$ the sum over $u\in C_i$ for $i$ fixed in the last expression, so that $d(\C, \C') = \sum_{i=1}^k d_i(\C')$.

Our sampling scheme will be done as follows.  Let $\eps$ be an error tolerance function, which we set below.   
For each cluster $C_i\in \C'$ and for each element $u\in C_i$ we will draw $k-i+1$ independent samples $S_{ui},S_{u(i+1)}, \dots, S_{uk}$ as follows.
Each sample $S_{uj}$ is a subset of $C_j$ of size $q$ (to be defined below), chosen uniformly with repetitions from $C_j$.
We will take
$$ q = c_2k^2\log n/\eps^4\ . $$ 
where $\delta$ is a failure probability (to be used below), and $c_2$ is a universal constant. 

Finally, we define our estimator $\hat f$ of $f$ to be:
$$  \hat f(\C') = \frac 1 2 \sum_{i=1}^k\frac {|C_i|}q \sum_{u\in C_i}\sum_{v\in S_{ui}} f_{u,v}(\C') + \sum_{i=1}^k\sum_{u\in C_i}\sum_{j=i+1}^k \frac {|C_j|} q\sum_{v\in S_{uj}} f_{u,v}(\C')\ .$$

Clearly for each $\C'$ it holds that $\hat f(\C')$ is an unbiased estimator of $f(\C')$.  We now analyze its error.  
For each $i,j\in [k]$ let $C_{ij}$ denote $C_i \cap C'_j$.  This captures exactly the set of elements in the $i$'th cluster
in $\C$ and the $j$'th cluster in $\C'$.  The distance $d(\C, \C')$ can be written as follows:
\begin{equation}\label{recdecomp}
 d(\C, \C') = \frac 1 2 \sum_{i=1}^k \sum_{j=1}^k |C_{ij} \times (C_i\setminus C_{ij})| + \sum_{j=1}^k \sum_{1 \leq i_1 < i_2 \leq k}|C_{i_1 j}\times C_{i_2 j}|\ .\end{equation}
We call each cartesian set product in (\ref{recdecomp}) a \emph{distance contributing rectangle}.
Note that unless a pair $(u,v)$ appears in one of the distance contributing rectangles, we have $f_{u,v}(\C') = \hat f_{u,v}(\C') = 0$.
Hence we can decompose $\hat f$ and $f$ in correspondence with the distance contributing rectangles, as follows:

\begin{align}
f(\C') &= \frac 1 2 \sum_{i=1}^k\sum_{j=1}^k  F_{i,j}(\C') + \sum_{j=1}^k \sum_{1 \leq i_1 < i_2 \leq k} F_{i_1,i_2, j} \label{recdecompf1}\\
\hat f(\C') &= \frac 1 2 \sum_{i=1}^k\sum_{j=1}^k  \hat F_{i,j}(\C') + \sum_{j=1}^k \sum_{1 \leq i_1 < i_2 \leq k} \hat F_{i_1,i_2, j} \label{recdecompf2}
\end{align}

where
\begin{align}
F_{i,j}(\C') &= \sum_{u\in C_{ij}}\sum_{v\in C_i\setminus C_{ij}} f_{u,v}(\C') \\
\hat F_{i,j}(\C') &= \frac{|C_i|} q \sum_{u\in C_{ij}}\sum_{v \in (C_i\setminus C_{ij})\cap S_{ui}} f_{u,v}(\C')  \label{hatFijdef}\\
F_{i_1, i_2, j}(\C') &= \sum_{u\in C_{i_1j}} \sum_{v\in C_{i_2 j}} f_{u,v}(\C') \\
\hat F_{i_1, i_2, j}(\C') &= \frac{|C_{i_2}|} q \sum_{u\in C_{i_1j}} \sum_{v\in C_{i_2 j}\cap S_{ui_2}} f_{u,v}(\C') 
\label{hatFi1i2jdef} 
\end{align}
(Note that the $S_{ui}$'s are multisets, and the inner sums in (\ref{hatFijdef}) and (\ref{hatFi1i2jdef}) may count elements
multiple times.)
\begin{lem}\label{lemrectangle1}
With probability at least $1-n^{-3}$, the following holds simultaneously for all $k$-clusterings $\C'$ and all $i,j\in [k]$:
\begin{equation}\label{approx1}
|F_{i,j}(\C') - \hat F_{i,j}(\C')| \leq \eps \cdot |C_{ij}\times (C_i\setminus C_{ij})| \ .
\end{equation}
\end{lem}
\begin{proof}
 Given a $k$-clustering $\C'=\{C'_1,\dots, C'_k\}$, the  predicate (\ref{approx1}) (for a given $i, j$) depends only on the set $C_{i j} = C_{i}\cap C'_j$.  Given a subset $B\subseteq C_i$, we say that $\C'$ $(i,j)$-realizes $B$ if $C_{ij} = B$.

Now fix $i,j$ and $B\subseteq C_i$.  Assume a $k$-clustering $(i,j)$-realizes  $B$.  Let $b=|B|$ and $c=|C_i|$.
Consider the random variable $\hat F_{ij}(\C')$ (see (\ref{hatFijdef})).  Think of the sample $S_{ui}$ as a sequence $S_{ui}(1), \dots, S_{ui}(q)$, where each $S_{ui}(s)$ is chosen uniformly at random from $C_i$ for $s=1,\dots, q$
We can now rewrite $\hat F_{ij}(\C')$ as follows:
$$ \hat F_{i,j}(\C') = \frac c q \sum_{u\in B}\sum_{s=1}^q X(S_{ui}(s))$$
where 
$$ X(v) = \begin{cases} f_{u,v}(\C') & v\in C_i\setminus C_{ij} \\ 0 & \mbox{otherwise} \end{cases}\ .$$
For all $s=1,\dots q$ the random variable $X(S_{ui}(s))$ is bounded by $2$ almost surely, and its moments satisfy:
\begin{align}
E[X(S_{ui}(s))] =& \frac 1 c \sum_{v\in(C_i\setminus C_{ij})} f_{u,v}(\C') \nonumber \\ 
E[X(S_{ui}(s))^2] \leq& \frac {4(c-b)} c\ .
\end{align}
From this we conclude using Bernstein inequality that for any $t \leq b(c-b)$,
\begin{equation*}
\Pr[|\hat F_{i,j}(\C') - F_{i,j}(\C') | \geq t] \leq \exp\left\{-\frac{qt^2}{16cb(c-b)}\right\}
\end{equation*}
Plugging in $t = \eps b(c-b)$, we conclude
\begin{equation*}
\Pr[|\hat F_{i,j}(\C') - F_{i,j}(\C') | \geq \eps b(c-b)] \leq \exp\left\{-\frac{q\eps^2 b(c-b)}{16c}\right\}
\end{equation*}
Now note that the number of possible sets $B\subseteq C_i$ of size $b$ is at most $n^{\min\{b, c-b\}}$.
Using union bound and recalling our choice of $q$, the lemma follows.

\end{proof}
The Lemma can be easily proven using the Bernstein probability inequality.  A bit more involved is the following:
\begin{lem}\label{lemrectangle2}
With probability at least $1-n^{-3}$, the following holds simultaneously for all $k$-clusterings $\C'$ and for all $i_1,i_2,j\in [k]$ with $i_1 < i_2$:

\begin{equation}\label{approx2}
|F_{i_1, i_2,j}(\C') - \hat F_{i_1,i_2,j}(\C')| \leq \eps \max \left \{|C_{i_1j} \times C_{i_2 j}|, \frac{ |C_{i_1 j} \times (C_{i_1}\setminus C_{i_1j})| }k, \frac {|C_{i_2 j}\times (C_{i_2}\setminus C_{i_2 j})|} k\right \}
\end{equation}
\end{lem}
\begin{proof}

 Given a $k$-clustering $\C'=\{C'_1,\dots, C'_k\}$, the  predicate (\ref{approx2}) (for a given $i_1,i_2, j$) depends only on the sets $C_{i_1 j} = C_{i_1}\cap C'_j$ and $C_{i_2 j} = C_{i_2}\cap C'_j$.  
Given subsets $B_1\subseteq C_{i_2}$ and $B_2\subseteq C_{i_2}$, we say that $\C'$ $(i_1, i_2,j)$-realizes $(B_1, B_2)$ if
$C_{i_1 j} = B_1$ and $C_{i_2 j} = B_2$.

We now fix $i_1<i_2, j$ 
and $B_1\subseteq C_{i_1}$, $B_2\subseteq C_{i_2}$.  Assume a $k$-clustering $\C'$ $(i_1, i_2, j)$-realizes $(B_1, B_2)$.  For brevity, denote $b_\iota=|B_\iota|$ and $c_\iota = |C_{i_\iota}|$ for $\iota=1,2$.
Using Bernstein inequality as before, we conclude that

\begin{equation}\label{bern1}\Pr[|F_{i_1, i_2,j}(\C') - \hat F_{i_1,i_2,j}(\C')| > t] \leq \exp\left \{-\frac {c_3 t^2  q}{b_1b_2c_2}\right \}\ .\end{equation}
for any $t$ in the range $\left [0, b_1b_2\right ]$, for some global $c_4>0$.
For $t$ in the range $(b_1b_2, \infty)$,
\begin{equation}\label{bern2}\Pr[|F_{i_1, i_2,j}(\C') - \hat F_{i_1,i_2,j}(\C')| > t] \leq \exp\left \{-\frac {c_5 t  q}{ c_2}\right \}\ .\end{equation}

\noindent We consider the following three cases.
\begin{enumerate}
\item $b_1b_2 \geq \max \{b_1(c_1-b_1/k, b_2(c_2-b_2)/k\}$.  Hence,
$b_1 \geq (c_2-b_2)/k, b_2\geq (c_1-b_1)/k$.  In this case, plugging in (\ref{bern1}) we get
\begin{align}
\Pr[|F_{i_1, i_2,j}(\C') - \hat F_{i_1,i_2,j}(\C')| > \eps b_1b_2] &\leq \exp\left \{-\frac {c_3 \eps^2 b_1 b_2  q}{c_2}\right \}  \ . \label{case1}
\end{align}
Consider two subcases. (i) If $b_2 \geq c_2/2$ then the RHS of (\ref{case1}) is at most $\exp\left \{-\frac {c_3 \eps^2 b_1  q}{2}\right \}$.  The number of sets $B_1,B_2$ of sizes $b_1, b_2$ respectively is clearly at most $n^{b_1+(c_2-b_2)} \leq n^{b_1 + kb_1}$. Therefore, if $q= O(\eps^{-2}k\log n)$, then with probability at least $1-n^{-6}$ simultaneously for all
$B_1, B_2$ of sizes $b_1, b_2$ respectively and for all $\C'$ $(i_1, i_2, j)$-realizing $(B_1, B_2)$ we have that
$ |F_{i_1, i_2,j}(\C') - \hat F_{i_1,i_2,j}(\C')| \leq \eps b_1 b_2\ .$  
(ii) If $b_2 < c_2/2$  then by our assumption, $b_1 \geq c_2/2k$.  Hence the RHS of (\ref{case1}) is at most $\exp\left \{-\frac {c_3 \eps^2 b_2  q}{2k}\right \}$.   The number of sets $B_1,B_2$ of sizes $b_1, b_2$ respectively is clearly at most $n^{(c_1-b_1)+c_2} \leq n^{b_2(1+k)}$.  Therefore, if $q= O(\eps^{-2}k^2\log n)$, then with probability at least $1-n^{-6}$ simultaneously for all $B_1, B_2$ of sizes $b_1, b_2$ respectively and for all $\C'$ $(i_1, i_2, j)$-realizing $(B_1, B_2)$ we have that
$ |F_{i_1, i_2,j}(\C') - \hat F_{i_1,i_2,j}(\C')| \leq \eps b_1 b_2\ .$  


\item $b_2(c_2-b_2)/k \geq \max\{b_1b_2, b_1(c_1-b_1)/k\}$.  We consider two subcases.
\begin{enumerate}
\item $\eps b_2(c_2-b_2)/k \leq b_1b_2$.  Using (\ref{bern1}), we get
\begin{equation}  \label{case2a} 
\Pr[|F_{i_1, i_2,j}(\C') - \hat F_{i_1,i_2,j}(\C')| > \eps b_2(c_2-b_2)/k] \leq \exp\left \{-\frac {c_3 \eps^2 b_2(c_2-b_2)^2  q} {k^2b_1c_2}\right \} 
\end{equation}
Again consider two subcases.  (i) $b_2 \leq c_2/2$.  In this case we conclude from (\ref{case2a}) 
\begin{equation}  \label{case2ai} 
\Pr[|F_{i_1, i_2,j}(\C') - \hat F_{i_1,i_2,j}(\C')| > \eps b_2(c_2-b_2)/k] \leq \exp\left \{-\frac{c_3 \eps^2 b_2c_2  q}{4k^2b_1} \right \} 
\end{equation}
Now note that by assumption 
\begin{equation}\label{abcd} b_1 \leq (c_2-b_2)/k \leq c_2/k \leq c_1/k\ . \end{equation} Also by assumption, 
$b_1 \leq b_2(c_2-b_2)/(c_1-b_1) \leq b_2c_2/(c_1-b_1)$.  Plugging in (\ref{abcd}), we conclude that $b_1 \leq b_2 c_2/(c_1(1-1/k)) \leq 2b_2c_2/c_1 \leq 2b_2$.  From here we conclude that the RHS of (\ref{case2ai}) is at most $\exp\left \{-\frac{c_3 \eps^2 2 c_2  q}{4k^2} \right \} $.
The number of sets $B_1,B_2$ of sizes $b_1, b_2$ respectively is clearly at most $n^{b_1+b_2} \leq n^{2b_2+b_2} \leq n^{3c_2}$.
Hence, if $q=O(\eps^{-2}k^2\log n)$  then with probability at least $1-n^{-6}$ simultaneously for all $B_1, B_2$ of sizes $b_1, b_2$ respectively and for all $\C'$ $(i_1, i_2, j)$-realizing $(B_1, B_2)$ we have that
$ |F_{i_1, i_2,j}(\C') - \hat F_{i_1,i_2,j}(\C')| \leq \eps b_2(c_2-b_2)/k\ .$  
In the second subcase (ii) $b_2 > c_2/2$.  The RHS of (\ref{case2a}) is at most $\exp\left \{-\frac {c_3 \eps^2 (c_2-b_2)^2  q} {2k^2b_1}\right \} $.  By our assumption, $(c_2-b_2)/(kb_1) \geq 1$, hence this is at most $\exp\left \{-\frac {c_3 \eps^2 (c_2-b_2)  q} {2k}\right \} $.  The number of sets $B_1,B_2$ of sizes $b_1, b_2$ respectively is clearly at most $n^{b_1+(c_2-b_2)} \leq n^{(c_2-b_2)/k + (c_2-b_2)} \leq n^{2(c_2-b_2)}$.  Therefore, if $q= O(\eps^{-2}k\log n)$, then with probability at least $1-n^{-6}$ simultaneously for all $B_1, B_2$ of sizes $b_1, b_2$ respectively and for all $\C'$ $(i_1, i_2, j)$-realizing $(B_1, B_2)$ we have that
$ |F_{i_1, i_2,j}(\C') - \hat F_{i_1,i_2,j}(\C')| \leq \eps b_2(c_2-b_2)/k\ .$  
\item  $\eps b_2(c_2-b_2)/k > b_1b_2$.   We now use (\ref{bern2}) to conclude
\begin{equation}  \label{case2b} 
\Pr[|F_{i_1, i_2,j}(\C') - \hat F_{i_1,i_2,j}(\C')| > \eps b_2(c_2-b_2)/k] \leq \exp\left \{-\frac {c_5 \eps b_2(c_2-b_2)  q} {kc_2}\right \} 
\end{equation}
We again consider the cases (i) $b_2 \leq c_2/2$ and (ii) $b_2 \geq c_2/2$ as above. In (i), we get  that  the RHS of (\ref{case2b}) is at most  $\exp\left \{-\frac {c_5 \eps b_2  q} {2k}\right \}$, that $b_1\leq 2b_2$ and hence the number of
possibilities for $B_1, B_2$ is at most $n^{b_1+b_2} \leq n^{3b_2}$.  In (ii), we get that  the RHS of (\ref{case2b}) is at most  $\exp\left \{-\frac {c_5 \eps (c_2-b_2)  q} {2k}\right \}$, and the number of possibilities for $B_1, B_2$ is at most
$n^{2(c_2-b_2)}$.  For both (i) and (ii) taking $q=O(\eps^{-1}k\log n)$ ensures with probability at least $1-n^{-6}$  simultaneously for all $B_1, B_2$ of sizes $b_1, b_2$ respectively and for all $\C'$ $(i_1, i_2, j)$-realizing $(B_1, B_2)$ we have that
$ |F_{i_1, i_2,j}(\C') - \hat F_{i_1,i_2,j}(\C')| \leq \eps b_2(c_2-b_2)/k\ .$ 
\end{enumerate}


\item $b_1(c_1-b_1)/k \geq \max\{b_1b_2, b_2(c_2-b_2)/k\}$.  We consider two subcases.
\begin{itemize}
\item $\eps b_1(c_1-b_1)/k \leq b_1b_2$.  
 Using (\ref{bern1}), we get
\begin{equation}  \label{case3a} 
\Pr[|F_{i_1, i_2,j}(\C') - \hat F_{i_1,i_2,j}(\C')| > \eps b_1(c_1-b_1)/k] \leq \exp\left \{-\frac {c_3 \eps^2 b_1(c_1-b_1)^2 q} {k^2b_2c_2}\right \} \ .
\end{equation}
As before, consider case (i) in which $b_2 \leq c_2/2$ and (ii) in which $b_2 \geq c_2/2$.  For case (i), we notice that the RHS of (\ref{case2b}) is at most $\exp\left \{-\frac {c_3 \eps^2 b_2(c_2-b_2)(c_1-b_1) q} {k^2b_2c_2}\right \}$ (we used the fact that $b_1(c_1-b_1) \geq b_2(c_2-b_2)$ by assumption).  This is hence at most $\exp\left \{-\frac {c_3 \eps^2 (c_1-b_1) q} {2k^2}\right \}$.  The number of possibilities of $B_1, B_2$ of sizes $b_1,b_2$ is clearly at most $n^{(c_1-b_1) + b_2} \leq n^{(c_1-b_1) + (c_1-b_1)/k} \leq n^{2(c_1-b_1)}$.  From this we conclude that $q=O(\eps^{-2} k^2\log n)$ suffices for this case.
For case (ii), we bound the RHS of (\ref{case3a}) by $\exp\left \{-\frac {c_3 \eps^2 b_1(c_1-b_1)^2 q} {2k^2b_2^2}\right \}$.
Using the assumption that $(c_1-b_1)/b_2 \geq k $, the latter expression is upper bounded  by 
$\exp\left \{-\frac {c_3 \eps^2 b_1 q} {2}\right \}$.  Again by our assumptions, 
\begin{equation}\label{gth} b_1 \geq b_2(c_2-b_2)/(c_1-b_1) \geq (\eps(c_1-b_1)/k)(c_2-b_2)/(c_1-b_1) = \eps(c_2-b_2)/k\ .\end{equation}
The number of possibilities of $B_1,B_2$ of sizes $b_1,b_2$ is clearly at most $n^{b_1 + (c_2-b_2)}$ which by (\ref{gth}) is bounded by $n^{b_1+kb_1/\eps} \leq n^{2kb_1/\eps}$.    From this we conclude that $q=O(\eps^{-3}k\log n)$ suffices for this case.

\item $\eps b_1(c_1-b_1)/k > b_1b_2$.  
\begin{equation}  \label{case3b} 
\Pr[|F_{i_1, i_2,j}(\C') - \hat F_{i_1,i_2,j}(\C')| > \eps b_1(c_2-b_1)/k] \leq \exp\left \{-\frac {c_5 \eps b_1(c_1-b_1)  q} {kc_2}\right \} 
\end{equation}
We consider two sub-cases, (i) $b_1 \leq c_1/2$  and (ii) $b_1 > c_1/2$.  
In case (i), we have that
\begin{align}
 \frac {b_1(c_1-b)}{c_2} & = \frac 1 2 \frac {b_1(c_1-b)}{c_2} + \frac 1 2 \frac {b_1(c_1-b)}{c_2}  \nonumber \\
&\geq \frac 1 2 \frac {b_1c_1}{2c_2} + \frac 1 2 \frac {b_2(c_2-b_2)}{c_2} \nonumber  \\
&\geq b_1/4 + \min\{b_2, c_2-b_2\}/2\ . \nonumber
\end{align}
Hence, the RHS of (\ref{case3b}) is bounded above by $\exp\left \{-\frac {c_5 \eps  q(b_1/4 + \min\{b_2, c_2-b_2\}/2)} {k}\right \}$.  The number of possibilities of $B_1,B_2$ of sizes $b_1,b_2$ is clearly at most $n^{b_1 + \min\{b_2, c_2-b_2\}}$, hence it suffices to take $q=O(\eps^{-1}k\log n)$ for this case.
In case (ii), we can upper bound the RHS of (\ref{case3b})  by $\exp\left \{-\frac {c_5 \eps c_1(c_1-b_1)  q} {2kc_2}\right \}  \geq \exp\left \{-\frac {c_5 \eps (c_1-b_1)  q} {2k}\right \} $.
The number of possibilities of $B_1, B_2$ of sizes $b_1, b_2$ is clearly at most $n^{(c_1-b_1) + b_2}$ which, using our assumptions, is bounded above by $n^{(c_1-b_1) + (c_1-b_1)/k} \leq n^{2(c_1-b_2)}$.  Hence, it suffices to take
$q=O(\eps^{-1}k\log n)$ for this case.
\end{itemize}
\end{enumerate}
This concludes the proof of the lemma.
\end{proof}

\noindent
As a consequence, we get the following:
\begin{lem}\label{epssmoothapprox}
with probability at least $1-n^{-3}$, the following holds simultaneously for all $k$-clusterings $\C'$:
$$ | f(\C') - f(\C) | \leq 3\eps d(\C', \C)\ .$$
\end{lem}
\begin{proof}
\begin{align}
| f(\C') - \hat f(\C') | = & \frac 1 2 \sum_{i=1}^k \sum_{j=1}^k |F_{i,j}(\C') - \hat F_{i,j}(\C')| 
+ \sum_{j=1}^k \sum_{1 \leq i_1 < i_2 \leq k}  |F_{i_1, i_2, j}(\C') - \hat F_{i_1, i_2, j}(\C')| \nonumber\\
\leq  &\frac 1 2 \sum_{i=1}^k \sum_{j=1}^k \eps^{-2} k^{-1} |C_{ij} \times (C_i \setminus C_{ij})| \nonumber  \\
& + \eps \sum_{j=1}^k \sum_{i_1 < i_2} \left ( |C_{i_1j}\times C_{i_2 j}| + k^{-1} |C_{i_1 j}\times (C_{i_1} \setminus C_{i_1 j}) | + k^{-1} |C_{i_2 j}\times (C_{i_2} \setminus C_{i_2 j}) | \right ) \nonumber \\
\leq & \frac 1 2 \sum_{i=1}^k \sum_{j=1}^k \eps^{-2} k^{-1} |C_{ij} \times (C_i \setminus C_{ij})|
           + \eps \sum_{j=1}^k \sum_{i_1 < i_2}  |C_{i_1j}\times C_{i_2 j}| \nonumber \\
& + \eps \sum_{j=1}^k  \sum_{i_1=1}^k \sum_{i_2=1}^k k^{-1} |C_{i_1 j}\times (C_{i_1} \setminus C_{i_1 j}) | + 
              \eps \sum_{j=1}^k  \sum_{i_2=1}^k \sum_{i_1=1}^k k^{-1} |C_{i_2 j}\times (C_{i_2} \setminus C_{i_2 j}) | \nonumber \\
\leq & \frac 1 2 \sum_{i=1}^k \sum_{j=1}^k \eps^{-2} k^{-1} |C_{ij} \times (C_i \setminus C_{ij})|
           + \eps \sum_{j=1}^k \sum_{i_1 < i_2}  |C_{i_1j}\times C_{i_2 j}| \nonumber \\
& + \eps \sum_{j=1}^k  \sum_{i_1=1}^k k k^{-1} |C_{i_1 j}\times (C_{i_1} \setminus C_{i_1 j}) | + 
              \eps \sum_{j=1}^k  \sum_{i_2=1}^k k k^{-1} |C_{i_2 j}\times (C_{i_2} \setminus C_{i_2 j}) | \nonumber \\
& \leq  \eps  \frac 3 2\sum_{i=1}^k \sum_{j=1}^k  |C_{ij} \times (C_i \setminus C_{ij})| +  \eps \sum_{j=1}^k \sum_{i_1 < i_2}  |C_{i_1j}\times C_{i_2 j}| \nonumber \\
&\leq 3\eps d(\C, \C') \nonumber
\end{align}
The first equality was (\ref{recdecompf1})-(\ref{recdecompf2}).  The second was Lemmas~\ref{lemrectangle1}-~\ref{lemrectangle2} (assuming success of a high probability event), the third, fourth and fifth inequalities were rearrangement of the sum, and
the final inequality came from (\ref{recdecomp}).
\end{proof}

\section{Conclusions and Future Work}
Our study considered the information theoretical problem of choosing which questions to ask in a game in which adversarially noisy
combinatorial pairwise information is input to a clustering algorithm.
We designed and analyzed a distribution from which drawing pairs is provably superior than the uniform distribution.
Our analysis did not take into account geometric information (e.g. a feature vector attached to each data point)
and treated the similarity labels as side information, as suggested in a recent line of literature.  It would be interesting to  study our solution in conjunction
with geometric information.  It would also be interesting to study our approach in the context of metric learning,
where the goal is to cleverly choose which pairs to obtain (noisy) distance labels for.  

\bibliographystyle{plainnat}
\bibliography{clustering}

\begin{thebibliography}{15}
\providecommand{\natexlab}[1]{#1}
\providecommand{\url}[1]{\texttt{#1}}
\expandafter\ifx\csname urlstyle\endcsname\relax
  \providecommand{\doi}[1]{doi: #1}\else
  \providecommand{\doi}{doi: \begingroup \urlstyle{rm}\Url}\fi

\bibitem[Ailon et~al.(2008)Ailon, Charikar, and Newman]{ACN08}
Nir Ailon, Moses Charikar, and Alantha Newman.
\newblock Aggregating inconsistent information: Ranking and clustering.
\newblock \emph{J. ACM}, 55\penalty0 (5):\penalty0 1--27, 2008.

\bibitem[Ailon et~al.(2011)Ailon, Begleiter, and Ezra]{AilonBE11}
Nir Ailon, Ron Begleiter, and Esther Ezra.
\newblock A new active learning scheme with applications to learning to rank
  from pairwise preferences.
\newblock In \emph{arXiv:1110.2136}, 2011.

\bibitem[Bansal et~al.(2002)Bansal, Blum, and
  Chawla]{Bansal02correlationclustering}
Nikhil Bansal, Avrim Blum, and Shuchi Chawla.
\newblock Correlation clustering.
\newblock In \emph{MACHINE LEARNING}, pages 238--247, 2002.

\bibitem[Bansal et~al.(2004)Bansal, Blum, and Chawla]{BBC04}
Nikhil Bansal, Avrim Blum, and Shuchi Chawla.
\newblock Correlation clustering.
\newblock \emph{Machine Learning}, 56:\penalty0 89--113, 2004.

\bibitem[Basu(2005)]{basu05}
Sugato Basu.
\newblock \emph{Semi-supervised Clustering: {P}robabilistic Models, Algorithms
  and Experiments}.
\newblock PhD thesis, Department of Computer Sciences, University of Texas at
  Austin, 2005.

\bibitem[Ben-Dor et~al.(1999)Ben-Dor, Shamir, and Yakhini]{Ben-DorSY99}
Amir Ben-Dor, Ron Shamir, and Zohar Yakhini.
\newblock Clustering gene expression patterns.
\newblock \emph{Journal of Computational Biology}, 6\penalty0 (3/4):\penalty0
  281--297, 1999.

\bibitem[Charikar and Wirth(2004)]{CharikarW04}
Moses Charikar and Anthony Wirth.
\newblock Maximizing quadratic programs: Extending grothendieck's inequality.
\newblock In \emph{FOCS}, pages 54--60. IEEE Computer Society, 2004.

\bibitem[Charikar et~al.(2005)Charikar, Guruswami, and Wirth]{CharikarGW05}
Moses Charikar, Venkatesan Guruswami, and Anthony Wirth.
\newblock Clustering with qualitative information.
\newblock \emph{J. Comput. Syst. Sci.}, 71\penalty0 (3):\penalty0 360--383,
  2005.

\bibitem[Cohn et~al.(2000)Cohn, Caruana, and
  Mccallum]{Cohn03semi-supervisedclustering}
David Cohn, Rich Caruana, and Andrew Mccallum.
\newblock Semi-supervised clustering with user feedback.
\newblock unpublished manuscript, 2000.
\newblock URL
  \url{http://www.cs.umass.edu/~mccallum/papers/semisup-aaai2000s.ps}.

\bibitem[Demiriz et~al.(1999)Demiriz, Bennett, and
  Embrechts]{Demiriz99semi-supervisedclustering}
Ayhan Demiriz, Kristin Bennett, and Mark~J. Embrechts.
\newblock Semi-supervised clustering using genetic algorithms.
\newblock In \emph{In Artificial Neural Networks in Engineering (ANNIE-99},
  pages 809--814. ASME Press, 1999.

\bibitem[Giotis and Guruswami(2006)]{GiotisGuruswami06}
Ioannis Giotis and Venkatesan Guruswami.
\newblock Correlation clustering with a fixed number of clusters.
\newblock \emph{Theory of Computing}, 2\penalty0 (1):\penalty0 249--266, 2006.

\bibitem[Klein et~al.(2002)Klein, Kamvar, and
  Manning]{Klein02frominstance-level}
Dan Klein, Sepandar~D. Kamvar, and Christopher~D. Manning.
\newblock From instance-level constraints to space-level constraints: Making
  the most of prior knowledge in data clustering.
\newblock In \emph{ICML}, pages 307--314, 2002.

\bibitem[Mitra and Samal(2009)]{5272169}
P.~Mitra and M.~Samal.
\newblock Approximation algorithm for correlation clustering.
\newblock In \emph{Networked Digital Technologies, 2009. NDT '09. First
  International Conference on}, pages 140 --145, july 2009.

\bibitem[Shamir et~al.(2004)Shamir, Sharan, and Tsur]{Shamir:2004:CGM}
Ron Shamir, Roded Sharan, and Dekel Tsur.
\newblock Cluster graph modification problems.
\newblock \emph{Discrete Applied Math}, 144:\penalty0 173--182, November 2004.

\bibitem[Xing et~al.(2002)Xing, Ng, Jordan, and Russell]{Xing02distancemetric}
Eric~P. Xing, Andrew~Y. Ng, Michael~I. Jordan, and Stuart Russell.
\newblock Distance metric learning, with application to clustering with
  side-information.
\newblock In \emph{Advances in Neural Information Processing Systems 15}, pages
  505--512. MIT Press, 2002.

\end{thebibliography}
\end{document}